%% file: ms.tex
\documentclass[letterpaper, 10 pt, conference]{ieeeconf}  

\IEEEoverridecommandlockouts                              

\RequirePackage{etex}
\usepackage{hyperref}
\usepackage{verbatim} 
\usepackage{amsmath} 
\usepackage{amssymb}  
\usepackage{ntheorem}
\usepackage{algorithm}
\usepackage{mathtools}
\usepackage{siunitx}
\usepackage{dblfloatfix}
\usepackage{tabularx}
\usepackage[usenames,dvipsnames]{xcolor}
\usepackage{psfrag,amsbsy,graphics,float}
\usepackage[dvips]{graphicx}
\usepackage{pgfplots}
\usepackage{pgfplotstable}
\usepackage{multirow}
\usepgfplotslibrary{fillbetween}
\usepackage{booktabs}
\pgfplotsset{compat = newest}
\usetikzlibrary{arrows,positioning,shapes,intersections,patterns,calc,external} 
\usepackage{cite}
\newcommand{\theinlineequation}{(\arabic{equation})}
 
\newcommand{\inlineeq}[1]{\refstepcounter{equation}\theinlineequation\ \(#1\)}
\newcolumntype{M}{R@{${}={}$}L}

\input{mydefs.tex}

\usepackage{autonum}

\title{\LARGE \bf
Mean Square Prediction Error of Misspecified\\ Gaussian Process Models
}

\author{Thomas Beckers, Jonas Umlauft and Sandra Hirche
\thanks{The authors are with the Chair of Information-oriented Control (ITR), Department of Electrical and Computer Engineering,
Technical University of Munich, 80333 Munich, Germany\newline
{\tt\small \{t.beckers, jonas.umlauft, hirche\}@tum.de}}
}

\begin{document}
  \setlength{\abovedisplayskip}{5pt}
\setlength{\belowdisplayskip}{5pt}
\setlength{\abovedisplayshortskip}{5pt}
\setlength{\belowdisplayshortskip}{5pt}

\maketitle
\thispagestyle{empty}
\pagestyle{empty}

\begin{abstract}
Nonparametric modeling approaches show very promising results in the area of system identification and control. A naturally provided model confidence is highly relevant for system-theoretical considerations to provide guarantees for application scenarios. Gaussian process regression represents one approach which provides such an indicator for the model confidence. However, this measure is only valid if the covariance function and its hyperparameters fit the underlying data generating process. In this paper, we derive an upper bound for the mean square prediction error of misspecified Gaussian process models based on a pseudo-concave optimization problem. We present application scenarios and a simulation to compare the derived upper bound with the true mean square error.
\end{abstract}

\section{Introduction}
Nonparametric or so-called data-driven models are an uprising modeling approach for the identification and control of systems with unknown dynamics. In contrast to classical parametric techniques, the idea is to let the data speak for itself without assuming an underlying, parametric model structure~\cite{pillonetto2014kernel}. Nonparametric models require only a minimum of prior knowledge for the regression of complex functions since the complexity of the model scales with the amount of training data~\cite{rasmussen2006gaussian}. Once a model of a system is learned from data, standard control laws such as model predictive control or feedback linerarization can be sucessfully applied~\cite{chowdhary2012model,umlauft:cdc2017}.\\
A general problem of data-driven models is the estimation of the model accuracy which is usually necessary for robust control design and stability considerations~\cite{zhou1998essentials}. For that reason, Gaussian process (GP) models are a promising nonparametric approach for control because they provide not only a mean prediction, but also a variance as uncertainty measure of the model. Specifically, a GP assigns to every point of an input space a normally distributed random variable. Any finite group of those random variables follows a multivariate Gaussian distribution, and in consequence there exists an analytic solution for the predicted mean and variance of a new test point.\\
The variance of the prediction is exploited in many different kinds of control approaches~\cite{medina2015synthesizing,beckers:cdc2017,kocijan2004gaussian}. However, the variance as prediction error measure is only valid if the GP model fits the data generating process, see~\cref{fig:CL}. 
\begin{figure}[t]
\begin{center}
\vspace{0.15cm}
	\input{figure/figure1_gp_miss.tex}
	\vspace{-0.25cm}\caption{The variance is misleading in terms of the model confidence.}\vspace{-0.8cm}
	\label{fig:CL}
\end{center}
\end{figure}
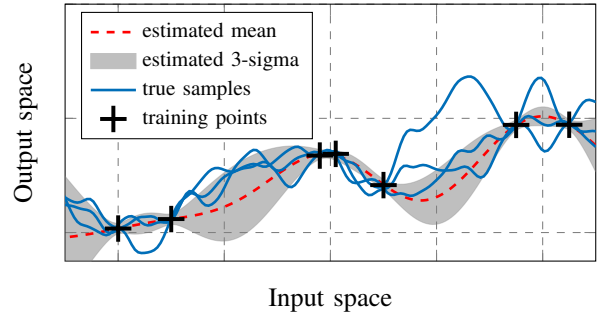
A GP model is fully described by a mean function, which is often set to zero~\cite{rasmussen2006gaussian}, and a covariance function. Although GPs with universal covariance functions often produce satisfactory results, the selection of a suitable covariance function is a nontrivial problem~\cite{seeger2000bayesian,pillonetto2011kernel}. In general, the problem is that only a finite data set is available to derive the covariance function. In addition, the covariance function typically depends on a number of hyperparameters. There exist many different methods to estimate these parameters based on the training data set, e.g. marginal likelihood optimization. However, the involved optimization problems are in general non-convex, such that the marginal likelihood may have multiple local optima~\cite{rasmussen2006gaussian}. Alternatively, there exists the cross validation approach which deals with a validation and test set to carry out the hyperparameter selection. Still, all of these methods do not guarantee that the covariance function and its hyperparameters fit the data generating process. As consequence, the variance of the GP model may not correctly estimate the real model confidence. A lower bound for the prediction error for GP models with a misspecified covariance is given by~\cite{waagberg2017prediction} whereas an upper bound is still missing. Using GP models in control, the upper bound is highly interesting for stability consideration based on robust control methods.\\
The contribution of this paper is the derivation of an upper bound for the mean square prediction error (MSPE) between an estimated GP model and a GP model with unknown covariance functions and hyperparameters. For this purpose, a set of possible covariance functions with corresponding hyperparameter sets must be given. We exploit the property that many commonly used covariance functions are pseudo-concave with respect to their hyperparameters. As consequence, the upper bound is the solution of pseudo-concave optimization problems. With additional assumptions, a closed form solution is provided.
\textbf{Notation:} Vectors are denoted with bold characters. Matrices are described with capital letters. The term~$A_{i,:}$ denotes the i-th row of the matrix~$A$. The expression~$\mathcal{N}(\mu,\Sigma)$ describes a normal distribution with mean~$\mu$ and covariance~$\Sigma$. The notation $\bm{a}\preccurlyeq\bm{b}$ describes the componentwise inequality between two vectors $a_i\leq b_i,\forall i$.
\section{Preliminaries and Problem Setting}\label{sec:pre}
\subsection{Gaussian Process Models}
Let~$(\Omega, \mathcal{F},P)$ be a probability space with the sample space~$\Omega$, the corresponding~$\sigma$-algebra~$\mathcal{F}$ and the probability measure~$P$. The index set is given by~$\X \subseteq \R^{n_x}$ with positive integer~$n_x$. Then, a function~$f(\x, \omega)$, which is a measurable function of~$\omega\in\Omega$ with~$\x\in\X$, is called a stochastic process and is simply denoted by~$f(\x)$. A GP is such a process which is fully described by a mean function~$m\colon\X\subseteq\R^{n_x}\to\R$ and a covariance function~$k\colon\Phi\times\X\times \X\to\R$ such that
\begin{align}
	&f(\x) \sim \GP(m(\x),k(\bm\varphi,\x,\x^\prime))
\end{align}
with the hyperparameter vector $\bm\varphi\in\Phi\subseteq\R^{l},\x,\x^\prime\in\X$. The mean function is usually defined to be zero, see~\cite{rasmussen2006gaussian}. The covariance function is a measure for the correlation of two states~$(\x,\x^\prime)$ and depends on hyperparameters $\bm\varphi$ whose number $l\in\N$ depends on the function used. A necessary and sufficient condition for the function $k(\cdot,\cdot,\cdot)$ to be a valid covariance function (denoted by the set $\mathcal{K}$) is that the Gram matrix is positive semidefinite for all possible input values~\cite{shawe2004kernel}. The choice of the covariance function and the determination of the corresponding hyperparameters can be seen as degrees of freedom of the regression. Probably the most widely used covariance function in Gaussian process modeling is the squared exponential (SE) covariance function, see~\cite{rasmussen2006gaussian}. An overview of the properties of different covariance functions can be found in~\cite{bishop2006pattern}.\\
In this paper, we use Gaussian process models with the assumption that the mean functions of the GPs are set to zero. Furthermore, a $n_x$-dimensional input space $\X$ and the output space $\R^{n_y}$ is considered, such that
\begin{align}
		\y=\f(\x)=
	\begin{cases} 
		f_1(\x)\sim \GP(0,k^1(\bm\varphi^1,\x,\x^\prime))\\
		\vdots\hspace{0.9cm}\vdots\hspace{0.5cm}\vdots\\
		f_{n_y}(\x)\sim \GP(0,k^n(\bm\varphi^n,\x,\x^\prime))
	\end{cases}\label{for:GPSSScov}
\end{align}
with $\x\in\X,\y\in\R^{n_y}$. The Gaussian process for each function $f_i$ depends on the covariance function $k^i$ with the set of hyperparameters~$\bm{\varphi}^i\in\Phi^i\subseteq\R^{l^i},l^i\in\N$ for all $i\in\{1,\ldots,n_y\}$. For the prediction, we concatenate~$m$ training inputs~$\{\x^j\}_{j=1}^m$ and training outputs~$\{\y^j\}_{j=1}^m$ in an input matrix~$X=[\x^1,\x^2,\ldots,\x^m]$ and a matrix of outputs~$Y^\top=[\y^1,\y^2,\ldots,\y^m]$ where~$Y_{i,:}$ are corrupted by Gaussian noise with variance $\sigma^2_{i}$. In summary, the training data for the Gaussian processes is described by~$\D=\{X,Y\}$. The joint distribution of the $i$-th component of $\y^*\in\R^{n_y}$ for a new test point $\x^*\in\X$ and the corresponding vector of the training outputs~$Y_{:,i}$ is given by
\begin{align}
	\begin{bmatrix} Y_{:,i} \\ y^*_{i} \end{bmatrix}\sim \mathcal{N} \left(\bm 0, \begin{bmatrix} K^i(\bm{\varphi}^i,X,X) & \bm k^i(\bm{\varphi}^i,\x^*,X)\\ \bm k^i(\bm{\varphi}^i,\x^*,X)^\top & k^i(\bm{\varphi}^i,\x^*,\x^*) \end{bmatrix}\right),\label{for:joint}
\end{align} 
where~$Y_{:,i}$ is the~$i$-th column of the matrix~$Y$. The function~$K^i\colon\Phi^i\times\X^m\times \X^m\to\R^{m\times m}$ is called the Gram matrix whose elements are $K^i_{j',j}= k^i(\bm{\varphi}^i,X_{:, j'},X_{:, j})+\delta(j,j')\sigma^2_i$ for all $j',j\in\{1,\ldots,m\}$. The delta function $\delta(j,j')=1$ for $j=j'$ and zero, otherwise, such that the variance~$\sigma^2_{i}$ is added to the diagonal of the Gram matrix. The vector-valued covariance function $\bm{k}^i\colon\Phi^i\times\X\times \X^m\to\R^m$, with the elements~$k^i_j = k^i(\bm{\varphi}^i,\x^*,X_{:, j})$ for all $j\in\{1,\ldots,m\}$, expresses the covariance between $\x^*$ and the input training data $X$. A prediction of~$y^*_i$ is derived from the joint distribution, see~\cite{rasmussen2006gaussian} for more details. This conditional probability distribution is Gaussian with the conditional mean
\begin{align}
	\mean_i(\y^*\vert \x^*,\D)&=\bm{k}^i(\bm{\varphi}^i,\x^*,X)^\top {K^i}^{-1}Y_{:,i},  \label{for:meanvalue}
\end{align}
and the predicted variance 
\begin{align}
	\var_i(\y^*\vert \x^*,\D)&=k^i(\bm{\varphi}^i,\x^*,\x^*)-\bm{k}^i(\bm{\varphi}^i,\x^*,X)^\top \notag\\
	& \phantom{{}=}{K^i}^{-1} \bm{k}^i(\bm{\varphi}^i,\x^*,X).\label{for:varvalue}
\end{align}
Based on~\cref{for:meanvalue} and~\cref{for:varvalue}, the $n_y$ normal distributed components~$y^*_i\vert \x^*,\D$ are combined in a multi-variable Gaussian distribution $\y^*\vert(\x^*,\D) \!\sim\! \mathcal{N} (\bm\mean(\cdot),\Var(\cdot))$
\begin{align}
	\bm \mean(\y^*\vert \x^*,\D)&=[\mean_1(\cdot),\ldots,\mean_{n_y}(\cdot)]^\top\label{for:meanvec}\\
	\Var(\y^*\vert \x^*,\D)&=\diag\left[\var_1(\cdot),\ldots,\var_{n_y}(\cdot)\right].\label{for:varvec}
\end{align} 
The hyperparameters~$\bm{\varphi}^i$ can be optimized by means of the likelihood function, thus by maximizing the probability of~$\bm{\varphi}^i = \arg\max_{\bm{\varphi}^i} \log P(Y_{:,i}|X,\bm{\varphi}^i)$ for all $i\in\{1,\ldots,{n_y}\}$.
\subsection{Problem Setting}
We consider two GP models $\GP^1,\GP^2$ following~\eqref{for:GPSSScov} each trained with the same set of data points $\mathcal{D}$. The model~$\GP^1$ is based on unknown covariance functions~$k^1,\ldots,k^{n_y}$ and hyperparameters~$\bm{\varphi}^1,\ldots,\bm{\varphi}^{n_y}$ whereas~$\GP^2$ uses the covariance functions~$\hat{k}^1,\ldots,\hat{k}^{n_y}$ and~$\hat{\bm{\varphi}}^1,\ldots,\hat{\bm{\varphi}}^{n_y}$. The goal is to compute the MSPE between the prediction $\y\in\R^{n_y}$ of~$\GP^1$ and the mean prediction of $\hat{\y}\in\R^{n_y}$ given by $\GP^2$, i.e.
\begin{align}
	\ev{\Verts{\y\vert(\x,\D)-\Mean(\hat{\y}\vert \x,\D)}^2}. \label{for:GPSSSes}
\end{align}
Since the covariance functions of $\GP^1$ are unknown, we derive an upper bound for the MSPE.
\begin{rem}
The reason for using the predicted mean of $\GP^2$ only is that we compare the MSPE with the predicted variance of $\GP^2$ to show that the variance can be misleading.
\end{rem}
In accordance with the no-free-lunch theorem, it is not possible to give error bounds for the MSPE without any assumptions on $k^1,\ldots,k^{n_y}$. Thus, we assume to have knowledge about a possible set of covariance functions $\tilde{\mathcal{K}}$ and a set of ranges for their hyperparameters $\tilde{\Phi}$.
\begin{assum}
\label{ass:pc}
Let $\tilde{\mathcal{K}}$ be a set of $z\in\N$ covariance functions 
\begin{align}
	\tilde{\mathcal{K}}=\{\tilde{k}^1,\ldots,\tilde{k}^{z}\in\mathcal{K}\}\label{for:covfcns}
\end{align}
which are positive and pseudo-concave with respect to their hyperparameters. In addition, let $\tilde\Phi$ be a set of convex sets
\begin{align}
	\tilde{\Phi}=\{\tilde{\Phi}^1,\ldots,\tilde{\Phi}^{z}|\tilde{\Phi}^j\subseteq\R^{l^j},l^j\in\N,j\in\{1,\ldots,z\}\}\label{for:hypset},
\end{align}
such that all elements of $\tilde{\Phi}^j$ are valid hyperparameters for~$\tilde{k}^j$, i.e. $\forall j\in\{1,\ldots,z\}, \tilde{k}^j\colon\tilde{\Phi}^j\times\X\times\X\to\R_{\geq 0}$. Then, there exists a function $\Psi\colon\{1,\ldots,{n_y}\}\to\{1,\ldots,z\}$, such that $k^i=\tilde{k}^{\Psi(i)},\bm{\varphi}^i\in\tilde{\Phi}^{\Psi(i)}$ for all $i\in\{1,\ldots,{n_y}\}$.
\end{assum}
Following this assumption, it is not necessary to know the exact covariance functions of $\GP^1$ but they must be elements of a set of possible covariance functions given by~$\tilde{\mathcal{K}}$. To keep this set as small as possible, statistical hypothesis testing could be used for discarding functions which are too unlikely. Analogously, the exact hyperparameters~$\bm{\varphi}^1,\ldots,\bm{\varphi}^{n_y}$ can be unknown but each of them is in a set of $\tilde{\Phi}^1,\ldots,\tilde{\Phi}^z$. In~\Cref{sec:pc}, we show that many common covariance functions are pseudo-concave and positive such as the squared exponential, the rational quadratic and the polynomial for specific inputs. A visualization of a possible configuration for the sets $\tilde{\mathcal{K}}$ and $\tilde\Phi$ is shown in~\cref{fig:set}.
\subsection{Application scenarios}
\textbf{Identification with GP state space models:} For learning an unknown dynamics, the GP state space model (GP-SSM) is a common choice in control~\cite{frigola2014variational}. Assuming a discrete-time system $\x_{\tau+1}\!=\!\f(\x_\tau)$ with $\x_\tau\!\in\R^{n_x},\f\colon\!R^{n_x}\!\to\! R^{n_x},\tau\!\in\!\N$. Based on the dynamics, a set of data $\D\!=\!\{\x_\tau,\x_{\tau+1}\}_{\tau=1}^m$ is generated. For the GP-SSM, the input space $\X$ is the space of current states $\x_\tau$ and the output space represents the predicted next step ahead states $\hat{\x}_{\tau+1}\in\R^{n_x}$, such that
\begin{align}
		\hat{\x}_{\tau+1}=\Mean(\hat{\x}_{\tau+1}\vert \x_\tau,\D)+\Var(\hat{\x}_{\tau+1}\vert \x_\tau,\D)\bm{\zeta}_\tau
\end{align}
with $\bm{\zeta}_\tau\sim\mathcal{N}(0,I)$. The predicted variance correctly represents the model uncertainty if the reproducing kernel Hilbert space norm~$\Verts{f_i}_{k^i}$ is bounded $\forall i\in\{1,\ldots,n_x\}$. This is not a strong limitation on the application side since universal covariance functions, e.g. the SE function, approximate any continuous function $f_i$ arbitrarily exactly on a closed set $\X$. However, without knowing the exact covariance function and hyperparameters, the predicted model uncertainty may not be correct. Our result~(\cref{thm:1,thm:2}) allows to derive an upper bound for the MPSE between the correct but unknown GP-SSM and an estimated GP-SSM. Consequently, the upper bound also captures the error between the estimated GP-SSM and the original discrete-time system.\\
\textbf{Reinforcement learning:} Following~\cite{engel2005reinforcement}, a Gaussian process model is used for the value process $V\colon\R^{n_x}\to\R$ which connects values and rewards in a reinforcement learning scenario. It includes the assumption that the choice of the covariance function reflects the prior concerning the correlation between the values of states and rewards. The presented~\cref{thm:1} can be used to avoid an eventually underestimated MSPE based on the predicted variance with suboptimal hyperparameters. In this scenario, the set $\tilde{\mathcal{K}}$ contains the selected covariance function $\tilde{k}^1$ only. Thus, an upper bound for the MSPE can be computed without knowing the exact hyperparameters. 
\vspace{-0.05cm}
\section{Mean square prediction error}
\label{sec:MSPE}
In this section, we present the computation of an upper bound for the MSPE between $\GP^1$ and the mean prediction of $\GP^2$ that is given by\footnote{For notational convenience we do not write the arguments $\x$ and $X$}
\vspace{-0.05cm}
\begin{align}
	&\ev{\Verts{\Delta}^2}=\sum_{i=1}^{n_y}  k^i(\bm{\varphi}^i)-2\hat{\bm{k}}^i(\hat{\bm{\varphi}}^i)^\top\hat{K}^{i^{-1}}(\hat{\bm{\varphi}}^i) \bm{k}^i(\bm{\varphi}^i)\notag\\
	&+\hat{\bm{k}}^i(\hat{\bm{\varphi}}^i)^\top\hat{K}^{i^{-1}}(\hat{\bm{\varphi}}^i) K^i(\bm{\varphi}^i)\hat{K}^{i^{-1}}(\hat{\bm{\varphi}}^i)\hat{\bm{k}}^i(\hat{\bm{\varphi}}^i)\label{for:mse0}
\end{align}
with error $\Delta=\y\vert(\x,\D)-\Mean(\hat{\y}\vert \x,\D)$. The covariance vector function $\hat{\bm{k}}$ and the Gram matrix $\hat{K}$ are related to $\GP^2$.
\begin{rem}	
	If the estimated covariance function and its hyperparameters are correct, i.e. $k^i=\hat{k}^i,\bm{\varphi}^i=\hat{\bm{\varphi}}^i$ for all $i$, the mean square error is simplified to
	\begin{align}
		\ev{\Verts{\Delta}^2}=\tr\left(\Var(\hat{\y}\vert \x,\D)\right),\label{for:postvarbound}
	\end{align}
	which is the trace of the posterior variance matrix.
\end{rem}
It is obvious, that the true covariance functions $k^i$ are needed to compute this error. To overcome this issue, we derive an upper bound based on a set of covariance functions and hyperparameters. For determining this bound, the maximum of~\eqref{for:mse0} has to be computed without knowing the covariance function $k^i$ and the corresponding hyperparameters $\bm{\varphi}^i$. With~\cref{ass:pc}, this problem is a non-convex, mixed-integer optimization problem. For simplicity in notation in the following derivations, parts of~\cref{for:mse0} are renamed as
\begin{align}
	\alpha^i(\x)&={k^i(\bm{\varphi}^i)}\label{for:maxprob1}\\
	\beta^i(\x)&={\hat{\bm{k}}^i(\hat{\bm{\varphi}}^i)^\top\hat{K}^{i^{-1}} \bm{k}^i(\bm{\varphi}^i)}\label{for:maxprob2}\\
	\gamma^i(\x)&={\hat{\bm{k}}^i(\hat{\bm{\varphi}}^i)^\top\hat{K}^{i^{-1}} K^i(\bm{\varphi}^i)\hat{K}^{i^{-1}}\hat{\bm{k}}^i(\hat{\bm{\varphi}}^i)}\label{for:maxprob3}
\end{align}
with $\alpha^i,\beta^i,\gamma^i\colon\X\to\R$.
\begin{lem}
	\label{lem:kmax}
	For any $k^i\in\tilde{\mathcal{K}}$, the inequality
	\begin{align}
	k^i(\bm{\varphi}^i,\x,\x^\prime)\leq\max_{j\in\{1,\ldots,z\}}\max_{\tilde{\bm{\varphi}}^j\in\tilde{\Phi}^j}{\tilde{k}^j(\tilde{\bm{\varphi}}^j,\x,\x^\prime)}\label{for:sol_kmax}
	\end{align}
	holds for~$\bm{\varphi}^i\in\tilde{\Phi}^{\Psi(i)},\forall i\in\{1,\ldots,n_y\}$ and~$\forall\x,\x^\prime\in\X$.
\end{lem}
\begin{proof}
	Since $k^i$ is an element of $\tilde{\mathcal{K}}$, the maximization over all covariance functions $\tilde{k}^j$ with their hyperparameter sets~$\tilde{\Phi}^j$ must be an upper bound for $k^i$. The optimization problem can be separated in an outer maximization over the finite number of covariance functions $\tilde{k}^j$ and an inner maximization over the convex hyperparameter sets.
\end{proof}
\begin{figure}[t]
\begin{center}
\vspace{0.15cm}
	\input{figure/figure2.tex}
	\vspace{-0.2cm}\caption{Example configuration for~\cref{ass:pc}}\vspace{-0.8cm}
	\label{fig:set}
\end{center}
\end{figure}
\begin{lem}
	\label{lem:maxprob2}
	Under~\cref{ass:pc}, there exists a lower bound $\ubar{\beta}^i(\x)\colon\X\to\R$ for~\cref{for:maxprob2} given by
	\begin{align}
	\ubar{\beta}^i(\x)&=\sum_{p=1}^m \min\left\lbrace h^i_p,0 \right\rbrace \max_{j}\max_{\tilde{\bm{\varphi}}^j}{\tilde{k}^j(\tilde{\bm{\varphi}}^j,\x,X_{:,p})}\phantom{a}  \label{for:sol_maxprob2}\\
	\bm{h}^i&=\hat{\bm{k}}^i(\hat{\bm{\varphi}}^i)^\top\hat{K}^{i^{-1}},\,\bm{h}^i\in\R^m
	\end{align}
	with $j\in\{1,\ldots,z\}$ and $\tilde{\bm{\varphi}}^j\!\in\!\tilde{\Phi}^j,\forall\x\in\X,\forall i\in\{1,\ldots,n_y\}$.
\end{lem}
\begin{proof}
	The term~\cref{for:maxprob2} can be lower bounded by
	\begin{align}
		\beta^i(\x)\geq\sum_{p=1}^m &\min\left\lbrace h^i_p,0 \right\rbrace \max \tilde{k}^j(\tilde{\bm{\varphi}}^j,\x,X_{:,p})\notag\\
		+&\max\left\lbrace h^i_p,0 \right\rbrace \min \tilde{k}^j(\tilde{\bm{\varphi}}^j,\x,X_{:,p})\label{for:minproof}\\
		&\text{ s.t. }j\in\{1,\ldots,z\},\tilde{\bm{\varphi}}^j\in\tilde{\Phi}^j,\notag
	\end{align}
	because the negative elements of $\bm{h}$ are multiplied with the maximum value of all covariance functions in $\tilde{\mathcal{K}}$ and vice versa. The minimum of $\tilde{k}^j(\tilde{\bm{\varphi}}^j,\x,X_{:,p})$ is always positive following~\cref{ass:pc}, so that 
	\begin{align}
		\beta^i(\x)\geq\sum_{p=1}^m &\min\left\lbrace h^i_p,0 \right\rbrace \max \tilde{k}^j(\tilde{\bm{\varphi}}^j,\x,X_{:,p})
	\end{align}
	s.t. $j\in\{1,\ldots,z\},\tilde{\bm{\varphi}}^j\in\tilde{\Phi}^j$ holds. With~\cref{lem:kmax}, we obtain the lower bound~\cref{for:sol_maxprob2}.
\end{proof}
\begin{lem}
	\label{lem:maxprob3}
	Under~\cref{ass:pc}, there exists an upper bound $\bar{\gamma}^i(\x)\colon\X\to\R$ for~\cref{for:maxprob3} given by
	\begin{align}
	\bar{\gamma}^i(\x)&\!\!=\!\!\sum_{\mathclap{p,q=1,\ldots,m}}\!\max\!\left\lbrace h^i_p h^i_q,\!0 \right\rbrace\!\max_{j}\max_{\tilde{\bm{\varphi}}^j}{\tilde{k}^j\!(\tilde{\bm{\varphi}}^j\!,\!X_{:,q},\!X_{:,p}\!)}  \label{for:sol_maxprob3}\\
	\bm{h}^i&\!=\!\hat{\bm{k}}^i(\hat{\bm{\varphi}}^i)^\top\hat{K}^{i^{-1}},\,\bm{h}^i\in\R^m
	\end{align}
	with $j\in\{1,\ldots,z\}$ and $\tilde{\bm{\varphi}}^j\!\in\!\tilde{\Phi}^j,\forall\x\in\X,\forall i\in\{1,\ldots,n_y\}$.
\end{lem}
\begin{proof}
	It is analogous to the proof of~\cref{lem:maxprob2}.
\end{proof}
\begin{thm}
	\label{thm:1}
	Consider the MSPE between the output~$\y$ of $\GP^1$ and the mean~$\hat{\y}$ of $\GP^2$~\cref{for:GPSSSes}. With~\cref{ass:pc}, there exists an upper bound for the MSPE given by
	\begin{align}
		\ev{\Verts{\Delta}^2}&\leq {n_y}\bar{\alpha}(\x)+\sum_{i=1}^{n_y} \bar{\gamma}^i(\x)-2\ubar{\beta}^i(\x)\label{for:thm1}\\
		\bar{\alpha}(\x)&=\max_{j\in\{1,\ldots,z\}}\max_{\tilde{\bm{\varphi}}^j\in\tilde{\Phi}^j}{\tilde{k}^j(\tilde{\bm{\varphi}}^j,\x,\x)}\label{for:sol_maxprob1}
	\end{align}
	with $\ubar{\beta}^i$ of~\cref{for:sol_maxprob2} and $\bar{\gamma}^i$ of~\cref{for:sol_maxprob3}.
\end{thm}
\begin{proof}
	The mean square error is upper bounded by the sum of the upper bounds for each term of~\cref{for:mse0}. An upper bound of~\cref{for:maxprob1} with~\cref{ass:pc} can be computed by~\cref{for:sol_maxprob1}	following~\cref{lem:kmax}. The bound~$\bar{\alpha}$ is independent of the training data~$\D$ and thus, independent of~$i$, so that it is summed up by ${n_y}\bar{\alpha}$. With~\cref{lem:maxprob2,lem:maxprob3}, the second and third term is bounded which results in~\cref{for:sol_maxprob1}. 
\end{proof}
\begin{rem}
The minimum of~\cref{for:minproof} is set to zero because the numerical computation would be hard to obtain since $\tilde{k}$ is only pseudo-concave. In this form, the solution of~(\ref{for:thm1}) can be computed by standard optimization algorithms~\cite{higgins1990minimizing}.
\end{rem}
\begin{table*}[!b]
	\caption{Pseudo-concave and componentwise monotonically increasing covariance functions\label{tab:prop}}
	\begin{tabularx}{\textwidth}{p{3cm}p{5cm}p{0.8cm}p{5cm}p{2cm}}
		\toprule
		Covariance  function & Expression $k(\bm{\varphi},\x,\x^\prime)=$ 		& &	Parameters & Domain	\\
		\midrule
		Polynomial & $(\x\x^\prime+\varphi^2)^p$ &\vspace{-0.23cm}\inlineeq{}\label{for:cov1}& $p\in\N,\varphi\in\R_{\geq 0}$ & $\forall\x,\x^\prime\in\R^{n_x}_{\geq 0}$\\
		Rational quadratic & $\varphi_2^2\left( 1+\frac{\Verts{\x-\x^\prime}^2}{2p \varphi_1^2}\right)^{-p}$ &\vspace{-0.23cm}\inlineeq{}\label{for:cov2}& $p\in\N_{>0},\bm{\varphi}\in\R^2_{>0}$&$\forall\x,\x^\prime\in\R^{n_x}$\\
		Squared exponential & $\varphi_{n_x+1}^2\negthickspace\exp\negthickspace\left(\negthickspace-\frac{(\x-\x^\prime)^\top P^{-1} (\x-\x^\prime)}{2}\negthickspace\right)$ &\vspace{-0.23cm}\inlineeq{}\label{for:cov3}& $P=\diag(\varphi_1^2,\ldots,\varphi_{n_x}^2),\bm\varphi\in\R^{n_x+1}_{>0}$ & $\forall\x,\x^\prime\in\R^{n_x}$\\
		Mat\'{e}rn & $\varphi_{2}^2 \frac{2^{1-\nu}}{\Gamma(\nu)}\!\left(\!\frac{\sqrt{2\nu} \Verts{\x-\x^\prime}}{\varphi_1}\!\right)^\nu\!\!\mathfrak{K}_\nu\!\left(\!\frac{\sqrt{2\nu} \Verts{\x-\x^\prime}}{\varphi_1}\!\right)$&\vspace{-0.23cm}\inlineeq{}\label{for:cov4}& $\nu=p+1/2,p\in\{0,1,2\},\bm\varphi\in\R^2_{\geq 0}$ & $\forall\x,\x^\prime\in\R^{n_x}$\\
	    \bottomrule
	\end{tabularx}
\end{table*}
\subsection{Closed form solution}
With additional assumptions, it is possible to provide a closed form solution for~\cref{for:thm1} of~\cref{thm:1}.
\begin{assum}
\label{ass:recset}
Each convex set of hyperparameters~$\tilde{\Phi}^j\in\tilde{\Phi}$ of~\cref{for:hypset} can be described by two vectors $\ubar{\bm \varphi}^j,\bar{\bm \varphi}^j\in\R^{l^j}$
\begin{align}
\tilde{\Phi}^j=\left\lbrace \tilde{\bm{\varphi}}^j\in\R^{l^j}|\ubar{\bm \varphi}^j\preceq\tilde{\bm{\varphi}}^j\preceq\bar{\bm \varphi}^j\right\rbrace,\forall j\in\{1,\ldots,z\}.
\end{align}
\end{assum}
\begin{assum}
\label{ass:inc}
Each covariance function $\tilde{k}^j,j\in\{1,\ldots,z\}$ of~\cref{for:covfcns} is componentwise strictly increasing with respect to its hyperparameters $\tilde{\bm\varphi}^j$, i.e. $\forall \tilde{\varphi}_i^j,\tilde{\upsilon}_i^j$ such that $\tilde{\varphi}_i^j<\tilde{\upsilon}_i^j$ one has~$\tilde{k}^j(\tilde{\bm\varphi}^j,\x,\x^\prime)\!<\!\tilde{k}^j(\tilde{\bm\upsilon}^j,\x,\x^\prime)$ $\forall\x,\x^\prime\in\!\X,\tilde{\bm\varphi}^j,\tilde{\bm\upsilon}^j\in\tilde{\Phi}^j$ for all $i\in\{1,\ldots,l^j\}$.
\end{assum}
\Cref{ass:recset} requires that each of the convex hyperparameter sets $\Phi^j\subseteq\R^{l^j}$ is a~$l^j$-dimensional hyperrectangle which is a weak restriction in practice. In~\Cref{sec:pc}, we show that~\cref{ass:inc} holds for some commonly used covariance functions. Based on these assumptions, there exists a closed form solution of~\cref{thm:1} because the maximum of the covariance function $\tilde{k}^j$ is now always at $\bar{\bm \varphi}^j$, see~\cref{fig:sekernel}. 
\begin{thm}
	\label{thm:2}
	Consider the MSPE between the output~$\y$ of $\GP^1$ and the mean~$\hat{\y}$ of $\GP^2$~\cref{for:GPSSSes}. With~\cref{ass:pc,ass:recset,ass:inc}, there exists an upper bound for the MSPE given by
	\begin{align}
		\ev{\Verts{\Delta}^2}&\leq \sum_{i=1}^{n_y} \max_j \left\lbrace\tilde{k}^j(\bar{\bm\varphi}^j,\x,\x)+\kappa^i(\x)-\eta^i(\x)\right\rbrace\\
		\eta^i(\x)&=2\sum_{p=1}^m \min\left\lbrace h^i_p,0 \right\rbrace  \tilde{k}^j(\bar{\bm{\varphi}}^j,\x,X_{:,p})\notag\\
				&+\max\left\lbrace h^i_p,0 \right\rbrace  \tilde{k}^j(\ubar{\bm{\varphi}}^j,\x,X_{:,p})\notag\\
		\kappa^i(\x)&=\sum_{\mathclap{p,q=1,\ldots,m}} \max\left\lbrace h^i_p h^i_q,0 \right\rbrace \tilde{k}^j(\bar{\bm{\varphi}}^j,X_{:,q},X_{:,p})  \notag\\
		 &+\min\left\lbrace h^i_p h^i_q,0 \right\rbrace \tilde{k}^j(\ubar{\bm{\varphi}}^j,X_{:,q},X_{:,p})\label{for:thm2}
	\end{align}
	with $\bm{h}^i=\hat{\bm{k}}^i(\hat{\bm{\varphi}}^i)^\top\hat{K}^{i^{-1}}$.
\end{thm}
\begin{rem}
The solution of~\cref{for:thm2} is a closed form expression in the sense that it can be evaluated in a finite number of operations because the maximization is over a finite set.
\end{rem}
\begin{proof}
	Assume that we choose $j\in\{1,\ldots,z\}$ of each maximization such that $\tilde{k}^j$ of~\cref{for:thm2} is equal to the covariance function $k^i$. With~\cref{ass:inc}, the covariance function $k^i$ with the hyperparameters $\bm{\varphi}^i$ is always equal or less then with $\bar{\bm\varphi}^i$ and vice versa, i.e. $k^i(\bm{\varphi}^i,\x,\x^\prime)\geq k^i(\ubar{\bm\varphi}^i,\x,\x^\prime)$. Thus, we have to prove
	\begin{align}
		\ev{\Verts{\Delta}^2}&\leq  \sum_{i=1}^{n_y} \big\{ k^i(\bar{\bm\varphi}^i,\x,\x)\notag\\
		&+ \sum_{\mathclap{\hspace{1cm}p,q=1,\ldots,m}} \max\left\lbrace h^i_p h^i_q,0 \right\rbrace k^i(\bar{\bm{\varphi}}^i,X_{:,q},X_{:,p})  \notag\\
		 &+\min\left\lbrace h^i_p h^i_q,0 \right\rbrace k^i(\ubar{\bm{\varphi}}^i,X_{:,q},X_{:,p})\notag\\
		 &-2\sum_{p=1}^m \min\left\lbrace h^i_p,0 \right\rbrace  k^i(\bar{\bm{\varphi}}^i,\x,X_{:,p})\notag\\
		&+\max\left\lbrace h^i_p,0 \right\rbrace  k^i(\ubar{\bm{\varphi}}^i,\x,X_{:,p})\big\},\label{for:pthm2}
	\end{align}
	 where each term of~\cref{for:pthm2} upper bounds the corresponding term of $\ev{\Verts{\Delta}^2}$ in~\cref{for:mse0} analogous to the idea in the proof of~\cref{lem:maxprob2}. Since~\cref{for:thm2} maximizes over all $\tilde{k}^j$ and, considering~\cref{ass:pc}, the covariance function $k^i$ is element of $\tilde{\mathcal{K}}$, there exists a $j$ such that the assumption at the beginning of the proof is fulfilled. \cref{for:cov1}
\end{proof}
\begin{cor}
\label{cor:1}
If $k^1=\cdots=k^{n_y}$ and $\bm{\varphi}^1=\ldots=\bm{\varphi}^{n_y}$, the closed form solution~\cref{for:thm2} of~\cref{thm:2} is equivalent to the posterior variance given by~\cref{for:postvarbound} for $\tilde{\mathcal{K}}=\{k^1\}$ and $\tilde\Phi=\{\bm{\varphi}^1\}$ 
\end{cor}
\begin{proof}
This is a result of~\cref{for:pthm2} if the set $\tilde{\mathcal{K}}$ only contains the covariance functions~$k^1=\cdots=k^{n_y}$ and the set $\tilde\Phi$ only the corresponding hyperparameters~$\bm{\varphi}^1=\ldots=\bm{\varphi}^{n_y}$.
\end{proof}
\begin{rem}
\Cref{cor:1} shows the convergence of the upper bound to the true MSPE~\cref{for:GPSSSes} between $\GP^1$ and $\GP^2$ for the minimum-size sets $\tilde{\mathcal{K}},\tilde\Phi$.
\end{rem}
\subsection{Pseudo-concave covariance functions}
In the following, we show that many common covariance functions fulfill~\cref{ass:pc,ass:inc}.
\label{sec:pc}
\begin{prop}
\label{prop:cf}
	The covariance functions \cref{for:cov1,for:cov2,for:cov3,for:cov4} with the corresponding parameters are pseudo-concave and componentwise monotonically increasing with respect to their hyperparameters on the designated domain.
\end{prop}
\begin{proof}
The following proof considers each covariance function separately.\\
\textbf{Polynomial:} The polynomial function $k$ is strictly increasing on $\varphi\in\R_{\geq 0}$ for any $\x,\x^\prime\in\R^{n_x}_{\geq 0}$ and hence, pseudo-concave~\cite{bazaraa2013nonlinear} and componentwise monotonically increasing.\\
\textbf{Rational quadratic:} The covariance function is quasi-concave if $\det H_3(\bm{\varphi})>0$ and $\det H_2(\bm{\varphi})<0$, where the matrix $H_r$ is the $r$-th order leading principal submatrix of the bordered Hessian of $k$ in respect to~$\bm\varphi$, see~\cite{bazaraa2013nonlinear}. The principal submatrices are given by
	\begin{align}
		H_2&=\frac{-4d^2p^2\varphi_2^4}{\varphi_1^2(2p\varphi_1^2+d)^2\left(\frac{2p\varphi_1^2+d}{2p\varphi_1^2}\right)^{2p}}<0,\\
		H_3&=\frac{8d\varphi_2^4p(dp+d+6p\varphi_1^2)}{\varphi_1^2(2p\varphi_1^2+d)^2\left(\frac{2p\varphi_1^2+d}{2p\varphi_1^2}\right)^{3p}}>0,
	\end{align}
	with $d\!=\!\Verts{\x-\x^\prime}^2\!>\!0,\forall p\in\N_{>0},\bm{\varphi}\in\R^2_{>0}$ so that the function is quasi-concave. Since $k\!\in\!\C^1$ and $\partial k\!/\!\partial \bm{\varphi}\!\neq\!\bm{0}$ on its domain, the function is also pseudo-concave~\cite{bazaraa2013nonlinear}. It is obviously also componentwise monotonically increasing.\\
\textbf{Squared exponential:} The covariance function can be rewritten as 
	\begin{align}
		k(\bm{\varphi},\x,\x^\prime)&=\exp\left(\log(\varphi_{n_x+1}^2)+\sum_{i=1}^{n_x} -\frac{\vert x_i-x_i^\prime\vert^2}{2\varphi_i^2}\right),
	\end{align}
	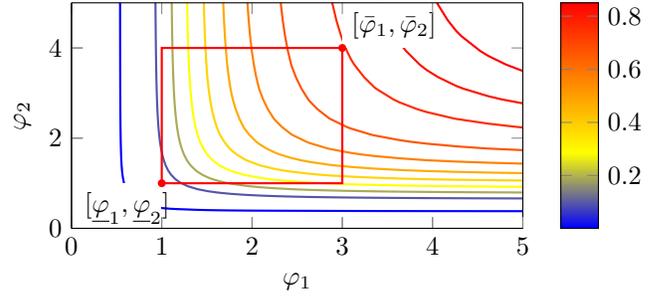
\begin{figure}
\begin{center}
\vspace{0.2cm}
	\input{figure/figure3.tex}
	\vspace{-0.8cm}\caption{The SE function $k(\bm\varphi,\x,\x)$ over its hyperparameters $\bm\varphi$. With \cref{ass:recset,ass:inc}, the maximum is at the corner of the hyperrectangle~$\bar{\bm \varphi}$.}\vspace{-0.8cm}
	\label{fig:sekernel}
\end{center}
\end{figure}where the argument of the exponential functions is quasi-concave, since this sum of concave functions is concave on all $\bm{\varphi}\in\R^{n_x+1}_{>0}$ for any $\x,\x^\prime\in\R^{n_x}$. The composition with the strictly increasing exponential function results in an overall quasi-concave function~\cite[Theorem 8.5]{sundaram1996first}. Since $k$ is continuous and $\partial k/\partial \bm{\varphi}\neq\bm{0}$ on its domain, the function is also pseudo-concave. Since the exponential and the logarithm function are monotonically increasing, the covariance function is componentwise monotonically increasing.\\
\textbf{Mat\'{e}rn:} For $p\in\N_{>0},\nu\!=\!p+1/2$, the function can be simplified to
    \begin{align}
    	k_{\bm{\varphi}}(d)\negthickspace=\negthickspace\varphi_{2}^2\exp\negthickspace\left(\negthickspace-\frac{\sqrt{2\nu}d}{\varphi_1}\negthickspace\right)\negthickspace \frac{p!}{(2p)!}\negthickspace\sum_{i=0}^p\negthickspace \frac{(p+i)!}{i!(p-i)!}\negthickspace{\left(\negthickspace\frac{\sqrt{8\nu}d}{\varphi_1}\negthickspace\right)\negthickspace.}^{\negthickspace p-i}
    \end{align}
   Analogous to the rational quadratic covariance, for the principal submatrices, it holds $\det\bar{H}_2<0$ and $\det\bar{H}_2>0$. With $k\in\C^1$ and $\partial k/\partial \bm{\varphi}\neq\bm{0}$ on its domain, the function is pseudo-concave. Since the exponential function grows faster than the polynomial, the covariance function is also componentwise monotonically increasing.
\end{proof}
\section{Simulation}\label{sec:sim}
In this section, we present a numerical example for the result of~\cref{thm:2} with GP-SSMs. For this purpose, we assume that a discrete-time, one-dimensional system can be correctly modeled by $\GP^1$ with Mat\'{e}rn covariance function where $p=1$ and the hyperparameters $\bm{\varphi}^1=[5.2,1.6]^\top$. The training set contains 10 uniformly distributed measurements.\\
Since the correct covariance function is usually unknown in real-world applications, the squared exponential (SE) covariance function is often used to learn the system dynamics. Following that approach, $\GP^2$ with SE covariance function is trained with the measurements of the system. 
The hyperparameters are optimized according to the likelihood function with a conjugate gradient method which results in the hyperparameters $\hat{\bm\varphi}^1=[0.36,0.32]^\top$.\\
In~\cref{fig:sim_mean}, the estimated mean~$\Mean(\hat{x}_{\tau+1}\vert x_\tau,\D)$ together with mean and variance of the true generating process are shown. It is obvious that the mean does not correspond to the true process, although it represents the training data effectively. 
As consequence, the mean square error between the estimated mean and the correct model is radically underestimated in the state space and in the time domain as presented in~\cref{fig:sim_mse}. To overcome this issue, we use~\cref{thm:2} to compute an upper bound of the MSPE without exact knowledge of the correct covariance function. For this purpose, we consider a set of covariance functions with their corresponding hyperparameter sets shown in~\Cref{tab:comp}. For comparison of different interval ranges, we use three different interval sizes around the true hyperparameters.~\Cref{fig:sim_mse} shows the estimated and true mean square prediction error which is normally unknown.
\begin{table}
	\caption{Covariance functions in $\tilde{\mathcal{K}}$ with hyperparameter sets in $\tilde\Phi$.\label{tab:comp}\vspace{-0.2cm}}
	\begin{tabularx}{\columnwidth}{p{4cm}p{5cm}}
		\toprule
		Covariance  functions		&	Hyperparameter sets	\\
		\midrule
		$\tilde{k}^1$: Mat\'{e}rn 	$p=1$  &	$\{0.9\bm{\varphi}^1\preceq\tilde{\bm{\varphi}}^1\preceq 1.1\bm{\varphi}^1\}$\\ 
		 &	$\{[0;0]\prec\tilde{\bm{\varphi}}^1\preceq 2\bm{\varphi}^1\}$\\ 
		&	$\{[0;0]\prec\tilde{\bm{\varphi}}^1\preceq 3\bm{\varphi}^1\}$\\ 
		$\tilde{k}^2$: Mat\'{e}rn 	$p=0,2$ &	$\{[1;1.5]\preceq\tilde{\bm{\varphi}}^2\preceq[10;2]\}$\\ 
		$\tilde{k}^3$: Rational quadratic 	$p=1$ &	$\{[1;0.1]\preceq\tilde{\bm{\varphi}}^3\preceq[20;1]\}$\rule{0pt}{2ex}\\
		$\tilde{k}^4$: Squared exponential  &	$\{[0.1;0.01]\preceq\tilde{\bm{\varphi}}^4\preceq[10;1]\}$\rule{0pt}{2ex}\\
	    \bottomrule
	\end{tabularx}
	\vspace{-0.2cm}
\end{table}	
The estimated error obviously underestimates the true MPSE. In contrast, the derived upper bound given by~\cref{thm:2} based on the functions of~\Cref{tab:comp} successfully confines the true MSPE. With a wider range of the interval the bound becomes loser.
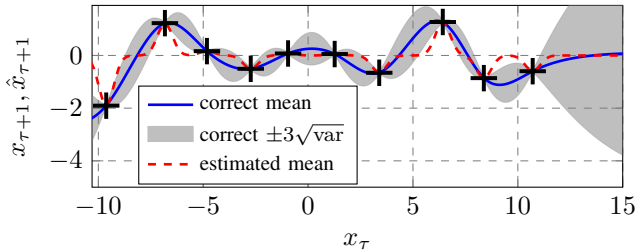
\begin{figure}[t]
	\begin{center}
		\input{figure/figure4_sim_mean.tex}
		\vspace{-0.7cm}\caption{Based on the training data, the estimated mean generates a misleading impression of the underlying process.}				\vspace{-0.8cm}
		\label{fig:sim_mean}
	\end{center}
\end{figure}
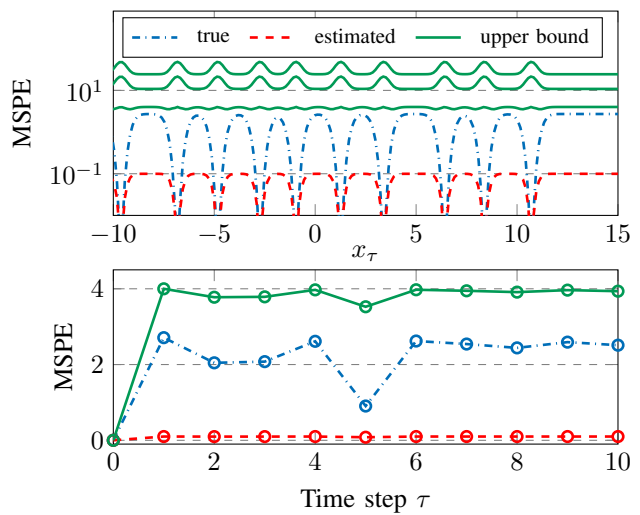
\begin{figure}[t]
	\begin{center}
	\vspace{0.13cm}
		\input{figure/figure5_sim_mse.tex}
		\vspace{-0.8cm}\caption{Top: The estimated, the true and the upper bound of the MSPE for a 10\%, 100\%, and 200\% error interval (from bottom to top) around the correct hyperparameter values. Bottom: The comparison in time domain with the 10\% bound.}				\vspace{-0.8cm}
		\label{fig:sim_mse}
	\end{center}
\end{figure}
\vspace{-0.0cm}
\section*{Conclusion}
We derive an upper bound for the mean square prediction error between an estimated GP model and a GP model with unknown covariance function. For the proposed upper bound, no exact knowledge about the underlying covariance function is required. Instead, only a set of possible covariance functions with their hyperparameter sets are necessary. With additional weak assumptions, a closed form solution is provided. A numerical example demonstrates that this bound confines the usually unknown mean square prediction error.

\section*{Acknowledgments}
The research leading to these results has received funding from the ERC Starting Grant ``Control based on Human Models (con-humo)'' agreement n\textsuperscript{o}337654.
\bibliography{mybib}
\bibliographystyle{ieeetr}

\end{document}

%% file: mydefs.tex
\newtheorem{rem}{Remark}
\newtheorem{thm}{Theorem}
\newtheorem{lem}{Lemma}
\newtheorem{cor}{Corollary}
\newtheorem{prop}{Proposition}
\newtheorem{assum}{Assumption}

\newcommand\tran{\mkern-2mu\raise1.25ex\hbox{$\scriptscriptstyle\top\hspace{0.5mm}$}\mkern-3.5mu}
\newcommand{\R}{\mathbb{R}}

\newcommand{\N}{\mathbb{N}}
\newcommand{\C}{\mathcal{C}}
\newcommand{\D}{\mathcal{D}}
\newcommand{\X}{\mathcal{X}}

\newcommand{\bm}[1]{{\boldsymbol{#1}}}
\newcommand{\Verts}[1]{{\left\Vert #1 \right\Vert}}
\DeclareMathOperator{\diag}{diag}
\DeclareMathOperator{\var}{var}
\DeclareMathOperator{\mean}{\mu}
\DeclareMathOperator{\Var}{\Sigma}
\DeclareMathOperator{\Mean}{\bm\mu}
\newcommand{\ev}[1]{\operatorname{E}\left[{#1}\right]}

\DeclareMathOperator{\tr}{Tr}
\newcommand{\GP}{\mathcal{GP}}
\newcommand{\x}{\bm x}
\newcommand{\f}{\bm{f}}
\newcommand{\y}[1][]{\bm{y}_{#1}}

\newcommand{\ubar}[1]{\underline{#1}}

\usepackage[noabbrev]{cleveref} 
\crefname{rem}{Remark}{Remarks}
\crefname{assum}{Assumption}{Assumptions}
\crefname{prop}{Proposition}{Propositions}
\crefname{cor}{Corollary}{Corollaries}
\crefname{lem}{Lemma}{Lemmas}
\crefname{thm}{Theorem}{Theorems}
\crefname{defn}{Definition}{Definitions}
\crefname{figure}{Fig.}{Fig.}
\Crefname{figure}{Figure}{Figures}
\crefname{equation}{}{}

%% file: figure/figure1_gp_miss.tex
\begin{tikzpicture}
\begin{axis}[
  xlabel={Input space},
  ylabel={Output space},
  legend pos=north west,
  grid style={dashed,gray},
  grid = both,
       width=\columnwidth,
  height=5cm,
  ymin=-0.5,
  ymax=4,
  xmin=-5,
  xmax=5,
  legend style={font=\footnotesize},
  legend cell align={left},
  xticklabels={,,},
  yticklabels={,,}]
\addplot+[name path=varp1, color=gray,opacity=0.3, no marks] table [x index=0,y expr=\thisrowno{1}+\thisrowno{2}]{data/figure1_model_GP.dat};
\addplot+[name path=varm1, color=gray,opacity=0.3, no marks] table [x index=0,y expr=\thisrowno{1}-\thisrowno{2}]{data/figure1_model_GP.dat};
\addplot[dashed,color=red,line width=1pt] table [x index=0,y index=1]{data/figure1_model_GP.dat};
\addplot[gray,opacity=0.5] fill between[ of = varm1 and varp1]; 
\addplot[color=NavyBlue,line width=1pt] table [x index=0,y index=3]{data/figure1_model_GP.dat};
\addplot[color=NavyBlue,line width=1pt] table [x index=0,y index=4]{data/figure1_model_GP.dat};
\addplot[color=NavyBlue,line width=1pt] table [x index=0,y index=5]{data/figure1_model_GP.dat};
\addplot[only marks,color=black,line width=1.5pt,mark=+,mark size=5] table [x index=0,y index=1]{data/figure1_points.dat};
\legend{,,estimated mean,estimated 3-sigma,true samples,,,training points};
\end{axis}
\end{tikzpicture} 

%% file: figure/figure2.tex
\begin{tikzpicture}[auto,>=latex',
    block/.style={
      rectangle,
      draw=black,
      text width=4.5em,
      font=\footnotesize,
      align=center,
      rounded corners
    }
]
\draw (-0.5,0) ellipse (1.2cm and 1.3cm);
\node[text width=1cm, align=center, text height=1.5ex, text depth=.25ex] (t) at (-0.5,1.5) {$\tilde{\mathcal{K}}$};
\node[draw=none, align=left,anchor=south west,inner sep=1pt] (a) at (-1.4,0.5) {\footnotesize $\tilde{k}^1=\text{SE-ARD}$};
\node[draw=none, align=left,anchor=south west,inner sep=1pt] (b) at (-1.4,-0.1) {\footnotesize $\tilde{k}^2=\text{POLY}$};
\node[draw=none, align=left,anchor=south west,inner sep=1pt] (c) at (-1.4,-0.7) {\footnotesize $\tilde{k}^3=\text{RQ}$};
\node[text width=4cm, align=center, text height=1.5ex, text depth=.25ex] at (-0.5,-1.7) {Covariance functions};
\draw (3.3,0) ellipse (2.2cm and 1.3cm);
\node[text width=1cm, align=center, text height=1.5ex, text depth=.25ex] (t1) at (3.3,1.5) {$\tilde\Phi$};
\node[block] (a1) at (3,0.8) {$[1,2]\times[0,1]$};
\node[text width=1.5cm, xshift=1.7cm,yshift=0.05cm] at (3,0.8) {$=\!\tilde{\Phi}^1$};
\node[block] (b1) at (3,0) {$[0,2]$};
\node[text width=1.5cm, xshift=1.7cm,yshift=0.05cm] at (3,0) {$=\!\tilde{\Phi}^2$};
\node[block] (c1) at (3,-0.8) {$[2,3]\times[0,1]$};
\node[text width=1.5cm, xshift=1.7cm,yshift=0.05cm] at (3,-0.8) {$=\!\tilde{\Phi}^3$};
\node[text width=4cm, align=center, text height=1.5ex, text depth=.25ex] at (3.3,-1.7) {Hyperparameter sets};
\draw [->, thick] (a.east) to (a1.west);
\draw [->, thick] (b.east) to (b1.west);
\draw [->, thick] (c.east) to (c1.west);
\end{tikzpicture}

%% file: figure/figure3.tex
\begin{tikzpicture}
    \begin{axis}[	height=3cm,
    				width=6cm,
    				view={0}{90}, 
    				xlabel={$\varphi_1$},
    				ylabel={$\varphi_2$},
    				xmin=0,xmax=5,ymin=0,ymax=5,
    				scale only axis,
    				colorbar]
    \addplot3[thick,mesh/rows=50,mesh/num points=2500, contour gnuplot={
    			levels={0.0010,0.0953,0.1897,0.2840,0.3783,0.4727,0.5670,0.6613,0.7557,0.81,0.8500},
    			output point meta=rawz,labels=false}] table [x index=0, y index=1, z index=2] {data/figure3_mesh.dat};
    \node[fill=white,text width=0.9cm, anchor=north east] at (1,1) {$[\underline{\varphi}_1,\underline{\varphi}_2]$};
    \node[fill=white,text width=0.9cm, anchor=south west] at (3,4) {$[\bar{\varphi}_1,\bar{\varphi}_2]$};
    \node[fill=red,circle,inner sep=0pt,minimum size=3pt] at (1,1) {};
    \node[fill=red,circle,inner sep=0pt,minimum size=3pt] at (3,4) {};
    \addplot[red,thick] coordinates {(1,1) (1,4) (3,4) (3,1) (1,1)};
    \end{axis}
\end{tikzpicture} 

%% file: figure/figure4_sim_mean.tex
\begin{tikzpicture}
\begin{axis}[
  xlabel={$x_\tau$},
  ylabel={$x_{\tau+1},\hat{x}_{\tau+1}$},
  grid style={dashed,gray},
  grid = both,
       width=\columnwidth,
  height=4cm,
  ymin=-5,
  ymax=1.9,
  xmin=-10.3,
  xmax=15,
  legend style={font=\footnotesize},
  legend style={at={(axis cs:-8.1,-4.9)},anchor=south west},
  legend cell align={left}]
\addplot+[name path=varp1, color=gray,opacity=0.3, no marks] table [x index=0,y expr=\thisrowno{1}+\thisrowno{2}]{data/figure4_model_GP.dat};
\addplot+[name path=varm1, color=gray,opacity=0.3, no marks] table [x index=0,y expr=\thisrowno{1}-\thisrowno{2}]{data/figure4_model_GP.dat};
\addplot[color=blue,line width=1pt] table [x index=0,y index=1]{data/figure4_model_GP.dat};
\addplot[gray,opacity=0.5] fill between[ of = varm1 and varp1]; 
\addplot[dashed,color=red,line width=1pt] table [x index=0,y index=1]{data/figure4_est_mean_tru_samples.dat};
\addplot[only marks,color=black,line width=1.5pt,mark=+,mark size=5] table [x index=0,y index=1]{data/figure4_tdata.dat};
\legend{,,correct mean,correct $\pm3\sqrt{\var}$,estimated mean};
\end{axis}
\end{tikzpicture} 

%% file: figure/figure5_sim_mse.tex
\begin{tikzpicture}
\begin{semilogyaxis}[
name=plot1,
legend columns=3, 
legend style={column sep=5pt},
  xlabel={$x_\tau$},
  x label style={at={(axis description cs:0.5,-0.1)},anchor=north},
  ylabel={MSPE},
  legend pos=north east,
  grid style={dashed,gray},
  ymajorgrids,
       width=0.96\columnwidth,
  height=4.3cm,
  ymin=0.01,
  ymax=800,
  xmin=-10,
  xmax=15,
  legend style={font=\footnotesize},
  legend cell align={left}]
\addplot[color=NavyBlue,dash dot,line width=1pt] table [x index=0,y index=1]{data/figure5_mse.dat};
\addplot[dashed,color=red,line width=1pt] table [x index=0,y index=2]{data/figure5_mse.dat};
\addplot[color=ForestGreen,line width=1pt] table [x index=0,y index=3]{data/figure5_mse.dat};
\addplot[color=ForestGreen,line width=1pt] table [x index=0,y index=4]{data/figure5_mse.dat};
\addplot[color=ForestGreen,line width=1pt] table [x index=0,y index=5]{data/figure5_mse.dat};
\legend{true, estimated, upper bound};
\end{semilogyaxis}
\begin{axis}[
name=plot2,
    at=(plot1.below south west), anchor=above north west,
legend columns=3, 
legend style={column sep=5pt},
  xlabel={Time step $\tau$},
  ylabel={MSPE},
  legend pos=north east,
  grid style={dashed,gray},
  ymajorgrids,
       width=0.96\columnwidth,
  height=3.9cm,
  ymin=-0.1,
  ymax=4.5,
  xmin=0,
  xmax=10,
  legend style={font=\footnotesize},
  legend cell align={left}]
\addplot[color=NavyBlue,dash dot,line width=1pt,mark=o,mark options={solid}] table [x index=0,y index=1]{data/figure6_mse_timeseries.dat};
\addplot[dashed,color=red,line width=1pt,mark=o,mark options={solid}] table [x index=0,y index=2]{data/figure6_mse_timeseries.dat};
\addplot[color=ForestGreen,line width=1pt,mark=o] table [x index=0,y index=3]{data/figure6_mse_timeseries.dat};
\end{axis}
\end{tikzpicture} 